\documentclass{article}
\usepackage{amsmath,amssymb,theorem}
\usepackage{verbatim}

\usepackage{color}
\usepackage{makeidx}

\theorembodyfont{\upshape}

 \makeatletter
 \@addtoreset{equation}{section}
 \makeatother

 \newtheorem{ittheorem}{Theorem}
 \newtheorem{itlemma}{Lemma}
 \newtheorem{itproposition}{Proposition}
 \newtheorem{itdefinition}{Definition}

 \newtheorem{itremark}{Remark}
 \newtheorem{itclaim}{Claim}
 \newtheorem{itcorollary}{\bf Corollary}

 \newenvironment{theorem}{\addtocounter{equation}{1}
 \begin{ittheorem}}{\end{ittheorem}}

 \newenvironment{lemma}{\addtocounter{equation}{1}
 \begin{itlemma}}{\end{itlemma}}

 \newenvironment{proposition}{\addtocounter{equation}{1}
 \begin{itproposition}}{\end{itproposition}}

 \newenvironment{definition}{\addtocounter{equation}{1}
 \begin{itdefinition}}{\end{itdefinition}}

 \newenvironment{remark}{\addtocounter{equation}{1}
 \begin{itremark}}{\end{itremark}}

 \newenvironment{claim}{\addtocounter{equation}{1}
 \begin{itclaim}}{\end{itclaim}}

 \newenvironment{proof}{\noindent {\bf Proof.\,}
 }{\hspace*{\fill}$\qed$\medskip}

 \newenvironment{corollary}{\addtocounter{equation}{1}
 \begin{itcorollary}}{\end{itcorollary}}

 \newcommand{\be}[1]{\begin{eqnarray*}\label{#1}}
 \newcommand{\ee}{\end{eqnarray*}}

 \newcommand{\bl}[1]{\begin{lemma}\label{#1}}
 \newcommand{\el}{\end{lemma}}

 \newcommand{\br}[1]{\begin{remark}\label{#1}}
 \newcommand{\er}{\end{remark}}

 \newcommand{\bt}[1]{\begin{theorem}\label{#1}}
 \newcommand{\et}{\end{theorem}}

 \newcommand{\bd}[1]{\begin{definition}\label{#1}}
 \newcommand{\ed}{\end{definition}}

 \newcommand{\bcl}[1]{\begin{claim}\label{#1}}
 \newcommand{\ecl}{\end{claim}}

 \newcommand{\bp}[1]{\begin{proposition}\label{#1}}
 \newcommand{\ep}{\end{proposition}}

 \newcommand{\bc}[1]{\begin{corollary}\label{#1}}
 \newcommand{\ec}{\end{corollary}}

 \newcommand{\bpr}{\begin{proof}}
 \newcommand{\epr}{\end{proof}}

 \newcommand{\bi}{\begin{itemize}}
 \newcommand{\ei}{\end{itemize}}

 \newcommand{\ben}{\begin{enumerate}}
 \newcommand{\een}{\end{enumerate}}

\def\un{{1\cdots 1}}

\def\zm{\{\,0,\dots,m\,\}}
\def\um{\{\,1,\dots,m\,\}}

\def\zul{\{\,0,1\,\}^\ell}

\def\zulm{\{\,0,1\,\}^\ell}

\def\zul{\{\,0,1\,\}^\ell}

\def\zulm{\big(\{\,0,1\,\}^\ell\big)^m} 
\def\uro{\smash{{U}^{\!\!\!\!\raise5pt\hbox{$\scriptstyle o$}}}}

\def\um{\{\,1,\dots,m\,\}}

\def\bp{{\overline{p}}}

\def\bp{{\overline{p}}}

 \def \ba {\begin{array}}
 \def \ea {\end{array}}

 \def \qed {{\heartsuit\hfill}}
 
 \def \R {{\mathbb R}}
 
 \def \N {{\mathbb N}}

 \def \cE {{\cal E}}

 \def \cB {{\cal B}}
 \def \cM {{\cal M}}

 \def \cP {{\cal P}}

\def \qed {{\square\hfill}}

 \def\cB{{\cal B}}  
\def\cE{{\cal E}}   
   
\def\cM{{\cal M}}   \def\cP{{\cal P}}

\def \qed {{\square\hfill}}

\def\R{{\mathbb R}}

\def\N{{\mathbb N}}

\def\eqref#1{(\ref{#1})}

\def\card{\text{card}\,}

\makeindex

\begin{document}

\title{The quasispecies regime for\\ the  simple genetic algorithm\\
with roulette--wheel selection}

 \author{
Rapha\"el Cerf
\\
DMA, 
{\'E}cole Normale Sup\'erieure\\
{{\it E-mail:} rcerf@math.u-psud.fr}
}

\maketitle



\begin{abstract}
\noindent
We introduce a new parameter to discuss the behavior of a genetic algorithm.
This parameter is the mean number of 
exact copies of the best fit chromosomes from one generation to the next.
We believe that the genetic algorithm operates best when this parameter
is slightly larger than~$1$ and we prove two results supporting this belief.
We consider the case of the simple genetic algorithm with the 
roulette--wheel selection mechanism. 
We denote by~$\ell$ the length of
the chromosomes, by $m$
the population size, by $p_C$ the  
crossover probability and by $p_M$ 
the mutation probability.
We start the genetic algorithm with an initial population whose maximal fitness
is equal to $f_0^*$ and whose mean fitness is equal to
${\overline{f_0}}$.
We show that,
in the limit of large populations,
the dynamics of the genetic algorithm depends in
a critical way on the parameter
$\pi
\,=\,\big({f_0^*}/{\overline{f_0}}\big)
(1-p_C)(1-p_M)^\ell\,.$
%
If $\pi<1$, then the genetic algorithm might operate
in a 
disordered regime: 
there exist positive constants $\beta$ and $\kappa$ which do not depend on $m$ such that,
for some fitness landscapes and some initial populations,
with probability
larger than $1-1/m^\beta$, before generation $\kappa\ln m$,
the best fit chromosome  will disappear,
and 
until generation $\kappa\ln m$,
the mean fitness will stagnate.
If $\pi>1$, then the genetic algorithm operates in a 
quasispecies regime: 
there exist positive constants $\kappa,p^*$ which do not depend on $m$ such that,
for any fitness landscape and any initial population,
with probability larger than $p^*$,
until generation $\kappa\ln m$,
the maximal fitness will not decrease and
before generation $\kappa\ln m$,
the mean fitness will increase by a factor $\sqrt{\pi}$.
These results suggest that 
the mutation and crossover probabilities should be tuned so that,
at each generation,
$
\text{maximal fitness}\times (1-p_C)(1-p_M)^\ell>
\text{mean fitness}
$.
\end{abstract}



\section{Introduction}
A central problem to implement efficiently a genetic
algorithm is the adjustment of the many parameters controlling
the algorithm.
If we focus on the classical simple genetic algorithm, these parameters
are: the population size, the probabilities of crossover and mutation.
There exists a huge literature discussing this question. 
The main message given by the numerous works conducted over the years
is that,
contrary to the initial hopes, there exists no universal choice of
parameters and the optimal choices depend
heavily on the
fitness landscape. We refer the reader to \cite{MFH}
for a recent review.

Our goal here is to attract the attention on a single parameter, which
somehow sums up the effects of the various mechanisms at work in a
genetic algorithm, and which is quite natural from the probabilistic
viewpoint. The parameter we have in mind is the mean number of 
exact copies of the best fit chromosomes from one generation to the next.
Let us call it~$\pi$.
We suggest that, at any generation, the various operators of the
genetic algorithm should be controlled in order to ensure that
$\pi$ is slightly larger than~$1$.
Indeed, if $\pi<1$, then the best fit chromosomes are doomed
to disappear
quickly from the population.
If $\pi>1$, then, with positive probability, the best fit chromosomes will 
perpetuate
and one of them will quickly become the
most recent common ancestor of the whole population.
It is not desirable that $\pi$ is much larger than~$1$, in order to avoid
the premature convergence of the algorithm.
The optimal situation is when the population retains the best fit chromosomes
and actively explores their neighborhoods. Ideally we would like to have
a few copies of the best fit chromosomes and a cloud of mutants descending
from them. 
This is why we aim at tuning the parameters so that $\pi$ is only
slightly larger than~$1$.
An interesting attempt to induce this behavior is what has been called "elitism"
in the genetic algorithm literature. Under elitism, the best fit chromosomes are 
automatically retained from one generation to the next. However, we believe that
the resulting dynamics is intrinsically different from the one we are aiming at
when tuning the parameters so that $\pi>1$. Indeed, we wish to build a 
probabilistic dynamics which automatically focuses the search around the
best fit chromosomes, and it might be that, even using elitism, the best fit
chromosomes are quickly forgotten during the search and none of them has a chance
to become the most recent common ancestor.

An advantage with the parameter $\pi$ is that we can easily compute simple bounds
in terms of the parameters of the algorithm. This becomes particularly true if 
we perform in addition an asymptotic expansion in one or several parameters.
If we do so, we can even prove rigorous results which strongly support the
previous ideas. More precisely, we will consider here the case of
large populations.
This kind of analysis has been previously conducted for the simple
genetic algorithm with ranking selection \cite{sga}. 
We try here to extend this analysis to the 
simple genetic algorithm with 
roulette--wheel selection. 
This task turned out to be very difficult, because the dynamics
is very sensitive to the variations of the fitness values.
Most of the results obtained for ranking selection do not
hold with roulette--wheel selection. We present only two results, which
demonstrate that, depending on the parameters and the fitness distribution
of the current population, the genetic algorithm can operate
either in a disordered regime, where the best fit chromosomes are typically
lost, or in a quasispecies regime, where the best fit chromosomes survive
and invade a positive fraction of the population.
Our results have their roots in
the quasispecies theory developed 
by Eigen, McCaskill
and Schuster \cite{ECS1}.
We refer the reader to the introduction of \cite{sga} for a quick 
summary of the
development of these ideas, as well as for pointers to the numerous
relevant references in the genetic algorithm literature.

There are several very interesting works which prove results related to ours, 
even in a more general context.
Lehre \cite{lehre_negative_2010, lehre_fitness-levels_2011},
Lehre and Dang \cite{dang_refined_2014},
Lehre and Yao \cite{lehre_impact_2012}
have succeeded in deriving upper bounds on the expected time to 
reach an optimal solution in a quite general framework covering a wide range
of population algorithms and objective functions. 
These results are more general and complex than ours. The results presented here do not yield any
estimate on the hitting time of the optimal solutions. 
Our goal is to emphasize the importance of the parameter~$\pi$ to understand the
behavior of the algorithm and its ability to take advantage of the best fit chromosomes present in
the population. To do so we consider only the case of the simple
genetic algorithm and we focus on its initial behavior in two contrasting situations.
In order to obtain a sharp and simple criterion, we rely also on asymptotic estimates, 
valid for large populations. Thus our results are much more specific than those
obtained in \cite{dang_refined_2014, 
lehre_negative_2010, lehre_fitness-levels_2011,
lehre_impact_2012},
yet they are in some sense sharper. 
An interesting project would be to analyze the relationship between~$\pi$ and the
quantities introduced in \cite{lehre_fitness-levels_2011},
like the cumulative selection probability~$\beta$
and the reproductive rate~$\alpha_0$.
Let us mention another related works.
Neumann, Oliveto and Witt \cite{NOW},
Oliveto and Witt \cite{OW1,OW2}
compute precise estimates describing the behavior of the simple genetic algorithm
with the OneMax function.
Eremeev \cite{Ere} derives
polynomial upper bounds for the hitting time of local maxima.
Recently, Corus, Dang, Eremeev and Lehre
\cite{corus_level-based_2014}
generalized the fitness--level technique to any variation operator
and obtained further bounds on well--known benchmark functions.

We study here the classical simple genetic algorithm with the roulette--wheel 
selection mechanism, as described in the famous books of
Holland \cite{HO} and Goldberg \cite{GO}.
We focus on the simplest genetic algorithm but we think that similar results might be proved
for variants of the algorithm. For instance, our results are not restricted to binary strings
and they hold for any finite alphabet. Similarly, we deal only with the one--point crossover, but our
results depend essentially on the probability $1-p_M$ of not having a crossover, thus they can
be readily extended to other crossover mechanisms.
We denote by $\ell$ the length of the chromosomes and by $m$ the population size.
We use roulette--wheel selection with replacement.
We use the standard single point crossover and the crossover probability
is denoted by~$p_C$.
We use independent parallel mutation at each bit
and the mutation probability
is denoted by~$p_M$.
We start the genetic algorithm with an initial population whose maximal fitness
is equal to $f_0^*$ and whose mean fitness is equal to
${\overline{f_0}}$.
We show that,
in the limit of large populations,
the dynamics of the genetic algorithm depends in
a critical way on the parameter
$$\pi
\,=\,\big({f_0^*}/{\overline{f_0}}\big)
(1-p_C)(1-p_M)^\ell\,.$$
%
%
$\bullet$ 
If $\pi<1$, then the genetic algorithm might operate
in a 
disordered regime: 
there exist positive constants $\beta$ and $\kappa$ which do not depend on $m$ such that,
for some fitness landscapes and some initial populations,
with probability
larger than $1-1/m^\beta$, before generation $\kappa\ln m$,
the best fit chromosome  will disappear and 
until generation $\kappa\ln m$,
the mean fitness will stagnate.

\noindent
$\bullet$ 
If $\pi>1$, then the genetic algorithm operates in a 
quasispecies regime: 
there exist positive constants $\kappa,p^*$ which do not depend on $m$ such that,
for any fitness landscape and any initial population,
with probability larger than $p^*$,
until generation $\kappa\ln m$,
the maximal fitness will not decrease and
before generation $\kappa\ln m$,
the mean fitness will increase by a factor $\sqrt{\pi}$.

\noindent
These results suggest that 
at each generation, the mutation and crossover probabilities should be tuned so that
$$
\text{maximal fitness}\times (1-p_C)(1-p_M)^\ell\,>\,
\text{mean fitness}
\,.$$
It seems therefore judicious to choose 
``large" values of~$p_M$ and~$p_C$ compatible with the condition~$\pi>1$.
In the generic situation where $f_0^*$ is significantly larger than 
${\overline{f_0}}$, 
this means that the mutation probability should be of 
order $1/\ell$; more precisely, the condition $\pi>1$ implies that
$$ \ell p_M +p_C \,<\, \ln
\big({f_0^*}/{\overline{f_0}}\big)
\,.$$

\section{The model}
In this section, 
we provide a brief description of the simple genetic algorithm.
The goal of the simple genetic algorithm is to find the global maxima
of a fitness function~$f$ defined on $\zul$ with values
in $]0,+\infty[$.
We consider the most classical and simple version of the genetic algorithm,
as described in Goldberg's book \cite{GO}. 
The genetic algorithm works with a population
of $m$ points of $\zul$, called the chromosomes, and it repeats
the following fundamental cycle in order to build the generation
$n+1$ from the generation~$n$:
\bigskip

\noindent
{\bf Repeat}
\medskip

\noindent
\qquad
$\bullet$ Select two chromosomes from the generation~$n$
\medskip

\noindent
\qquad
$\bullet$ Perform the crossover
\medskip

\noindent
\qquad
$\bullet$ Perform the mutation
\medskip

\noindent
\qquad
$\bullet$ Put the two resulting chromosomes in generation~$n+1$
\medskip

\noindent
{{\bf Until} there are $m$ chromosomes in generation~$n+1$}
\bigskip

\noindent
When building the generation~$n+1$ from the generation~$n$, 
all the random choices are performed independently. 
We use the classical genetic operators, as in Goldberg's book \cite{GO},
which we recall briefly.
\bigskip

\noindent
{\bf Selection.} 
We use roulette--wheel selection with replacement.
The probability of selecting the $i$--th chromosome
$x(i)$
in the population~$x$ is given by the selection distribution defined by
$$P\big(\text{select $i$--th chromosome in $x$}\big)\,=\,\frac{f(x(i))}{f(x(1))+\cdots+f(x(m))}\,.$$ 

\noindent
{\bf Crossover.} 
We use the standard single point crossover and the crossover probability
is denoted by~$p_C$:
\medskip

\noindent
$$P\Bigg(
\lower 7pt\hbox{
\vbox{
\hbox{$000\raise 3pt\hbox{\vrule height 0.4pt width 25pt}011$}
\hbox{$100\raise 3pt\hbox{\vrule height 0.4pt width 25pt}110$}
}
\hbox{\kern-2pt\vrule width 0.4pt height 25pt depth 10pt\kern-2pt}
\vbox{
\hbox{$011\raise 3pt\hbox{\vrule height 0.4pt width 25pt}001$}
\hbox{$001\raise 3pt\hbox{\vrule height 0.4pt width 25pt}111$}
}
\raise 7pt\hbox{$\longrightarrow$}\!\!
\hbox{
\vbox{
\hbox{$000\raise 3pt\hbox{\vrule height 0.4pt width 25pt}011$}
\hbox{$100\raise 3pt\hbox{\vrule height 0.4pt width 25pt}110$}
}
\hbox{\kern-2pt\vrule width 0.4pt height 25pt depth 10pt\kern-2pt}
\vbox{
\hbox{$001\raise 3pt\hbox{\vrule height 0.4pt width 25pt}111$}
\hbox{$011\raise 3pt\hbox{\vrule height 0.4pt width 25pt}001$}
}}}\Bigg)\,=\,
\frac{p_C}{\ell-1}\,.$$
\smallskip

\noindent
{\bf Mutation.} 
We use independent parallel mutation at each bit
and the mutation probability
is denoted by~$p_M$:
$$P\Big(0000000\longrightarrow 0101000\Big)\,=\,p_M^2(1-p_M)^5\,.$$
%
%
\section{The results}
We denote by $x_0$ the initial population
and by $x_0(1),\dots,x_0(m)$ the $m$ chromosomes in the population~$x_0$.
We denote by $f^*_0$ the maximal fitness of the chromosomes in $x_0$
and by 
${\overline{f_0}}$ their mean fitness, i.e.,
$$f^*_0\,=\,\max_{1\leq i\leq m}f(x_0(i))\,,
\qquad
{\overline{f_0}}\,=\,\frac{1}{m}
\sum_{1\leq i\leq m}f(x_0(i))\,.$$
We present two results to illustrate the contrasting behavior of the genetic
algorithm when $\pi<1$ and when $\pi>1$.
\medskip

\noindent
{\bf The disordered regime.}
We consider the fitness function $f$ defined by
$$\forall u\in\zul\qquad
f(u)\,=\,
\begin{cases}
2
&\text{if }u=1\cdots 1\\
1
&\text{otherwise }\\
\end{cases}
$$
This corresponds to the sharp peak landscape.
The chromosome $1\cdots 1$ is called the Master sequence.
We start the genetic algorithm from the population
$x_0$ containing one Master sequence $1\cdots 1$ and $m-1$ copies
of the chromosome $0\cdots 0$.
Thus the optimal chromosome is already present in the population. Our goal is to study its influence on
the evolution of the population.
This is a crude model for the following scenario: the genetic algorithm has been stuck for a
long time, and suddendly, by chance, a chromosome with a superior fitness is found;  
is this new chromosome likely to influence the whole population or will it disappear?
The next theorem describes a situation where the mean fitness of the population
is unlikely to increase despite the presence of a very well fit chromosome.
\begin{theorem}\label{thspl}
Let $\pi<1$ be fixed.
We suppose that the parameters 
are set so that
$\ell=m$ and
$\big({f_0^*}/{\overline{f_0}}\big)
(1-p_C)(1-p_M)^\ell=\pi$.
There exist strictly
positive constants $\kappa,\beta,m_0$, which depend on $\pi$
only, such that, for the 
genetic algorithm
starting from $x_0$, for any $m\geq m_0$,
$$ P\left(
\begin{matrix}
\text{before generation $\kappa\ln m$,
the Master sequence disappears}\\
\text{until generation $\kappa\ln m$,
the mean fitness is $\leq$ $\overline{f_0}(1+
\frac{1}{\sqrt{m}})$}
\\
\end{matrix}
\right)
\geq\,1-\frac{1}{m^\beta}\,.
$$
\end{theorem}

\noindent {\bf The quasispecies regime.}
We consider an arbitrary non--negative fitness function $f$
and we start the genetic algorithm from a population~$x_0$ such that
${f_0^*}>{\overline{f}_0}$.
The next theorem describes a situation where the mean fitness of the population
is likely to increase thanks to the influence of the best fit chromosome.
\begin{theorem}\label{thqsr}
Let $\pi>1$ be fixed.
We suppose that the parameters 
are set so that
$\big({f_0^*}/{\overline{f}_0}\big)
(1-p_C)(1-p_M)^\ell=\pi$.
There exist strictly
positive constants $\kappa,p^*$, which depend on $\pi$ and
the ratio ${f_0^*}/{\overline{f}_0}$
only, such that, 
for the 
genetic algorithm
starting from $x_0$,
for any $\ell,m\geq 1$,
$$ P\left(
\begin{matrix}
\text{until generation $\kappa\ln m$,
the maximal fitness is always $\geq f_0^*$}\\
\text{before generation $\kappa\ln m$,
the mean fitness becomes $\geq \sqrt{\pi}\,\overline{f_0}$}
\\
\end{matrix}
\right)
\,\geq\,p^*\,.
$$
\end{theorem}

\section{The disordered regime}\label{secdis}
In this section, we will prove theorem~\ref{thspl}.
The proof has two main steps. First we define a process $(T_n)_{n\in\mathbb N}$
which counts the number of descendants of the Master sequence in generation~$n$.
We show that, as long as $T_n\leq m^{1/4}$, the process
$(T_n)_{n\in\mathbb N}$ is stochastically dominated by a supercritical Galton--Watson
process.
Next we define a process $(N^*_n)_{n\in\mathbb N}$ 
which counts the number of Master sequences present in generation~$n$.
Note that $N^*_n$ is in general smaller than $T_n$, because of the mutations
and the crossovers. Indeed a chromosome might have an ancestor which
is a Master sequence and be very different from it.
We show then that, as long as $T_n\leq m^{1/4}$, the process
$(N^*_n)_{n\in\mathbb N}$ is stochastically dominated by a subcritical Galton--Watson
process. The bound on
$(N^*_n)_{n\in\mathbb N}$ relies on the previous bound on
$(T_n)_{n\in\mathbb N}$.
We finally invoke a classical argument from the theory of branching processes
to prove that this subcritical Galton--Watson process becomes extinct 
before generation $\kappa\ln m$ with probability larger than
$1-{1}/{m^\beta}$. 
The computations are tedious, because we need to control the probabilities
of obtaining a Master sequence when applying the various genetic operators,
and the crossover creates correlations between pairs of adjacent chromosomes.

Let us start with the precise proof.
We start the genetic algorithm from the population
$x_0$ containing one Master sequence $1\cdots 1$ and $m-1$ copies
of the chromosome $0\cdots 0$.
Let $\pi<1$ be fixed. Throughout the proof, we suppose that
$\ell,p_C,p_M$ satisfy
$\ell=m$ and 
$$2(1-p_C)(1-p_M)^\ell\,=\,\pi\,.$$
We denote by $X_n$ the population at generation~$n$ and
by $T_n$ 
the number of descendants of the initial Master sequence present in~$X_n$.
To build the generation~$n+1$, we select (with replacement) $m$ chromosomes
from the population~$X_n$. Let us denote by $A_n$
the number of chromosomes 
selected in~$X_n$ which are a descendant of the initial Master sequence.
Each of these chromosomes is the parent of two chromosomes in generation~$n+1$ 
(because
of the crossover operator). Thus we can bound $T_{n+1}$ from above by $2A_n$.
Conditionally on $T_n$, the distribution of $A_n$
is binomial with parameters $m$ and
$$
\frac{2T_{n}}{2T_n+m-T_{n}}\,\leq\,
\frac{2T_{n}}{m}
\,.$$
Thus, 
conditionally on $T_n$, the distribution of $T_{n+1}$
is stochastically dominated by the binomial distribution
$2\cB(m,2T_n/m)$, which we write
$$T_{n+1}\,\preceq\,
2\cB\Big({m},
\frac{2}{m}
T_{n}\Big)\,.$$
The symbol $\preceq$ means stochastic domination 
(see the appendix). 
We define
$$\tau_1\,=\,\inf\,\big\{\,n\geq 1: T_n>m^{1/4}\,\big\}\,,$$
and we will compute estimates which hold until time $\tau_1$.
So we fix $n\geq 1$ and we condition on the event that $\tau_1>n$.
There exists $t_0>0$ such that, for $0<t<t_0$, we have
$\ln(1-t)\geq -2t$. Therefore, for $m$ large enough so that 
$2 m^{-3/4}<t_0$, we have
$$
\Big(1-
\frac{2}{m}
T_{n}1_{\{\,\tau_1> n\,\}}\Big)^{m}\,\geq\,
\exp\Big(
-
{4}
T_{n}1_{\{\,\tau_1>n\,\}}\Big)
\,.$$
We denote by $\cP(\lambda)$ the Poisson law of parameter~$\lambda$.
By lemma~\ref{binopoi}, we conclude from this inequality that
$$2\cB\Big({m},
\frac{2}{m}
T_{n}1_{\{\,\tau_1> n\,\}}\Big)\,\preceq\,
2\cP\big(4
T_{n}1_{\{\,\tau_1> n\,\}}\big)
\,.$$
Therefore,
$$\displaylines{
T_{n+1}1_{\{\,\tau_1\geq n+1\,\}}\,\preceq\,
\sum_{k=1}^{
T_{n}1_{\{\,\tau_1> n\,\} }
}
V_k\,,
}$$
where the random variables
$(V_k)_{k\geq 1}$ are independent identically distributed
with distribution twice the
Poisson law of parameter~$4$.
Let 
$(Z_n)_{n\in\mathbb N}$
be a Galton--Watson process starting from $Z_0=1$
with reproduction law
$2\cP(4)$.
We conclude from the previous inequality that
$$\forall n\geq 0\qquad
T_n1_{\{\,\tau_1\geq n\,\}}\,\preceq\,Z_n\,.$$
We denote by 
$X_n(1),\dots,X_n(m)$ the $m$ chromosomes of the population~$X_n$.
Let $N^*_n$ be the number of Master sequences present in the 
population at time~$n$:
$$\forall n\geq 0\qquad
N^*_n\,=\,\card\,\big\{\,i\in\um:X_n(i)=1\cdots 1\,\big\}\,.$$
We want to control
$N^*_{n+1}$ conditionally on the knowledge of $N^*_n$ and $T_n$.
A difficulty is that the crossover operator creates
correlations between the chromosomes of $X_{n+1}$.
However, conditionally on $X_n$, the pairs of consecutive chromosomes
$$(X_{n+1}(1), X_{n+1}(2)),\dots,
(X_{n+1}(m-1), X_{n+1}(m))$$
are i.i.d.. Therefore, we can write 
$N^*_{n+1}$ as the sum
$$
N^*_{n+1}\,=\,\sum_{i=1}^{m/2}Y_i\,,$$
where $Y_i$ is the number of Master sequences in the $i$--th pair
$(X_{n+1}(2i-1), X_{n+1}(2i))$.
Our strategy consists in estimating the 
conditional distribution of the $Y_i$'s, knowing the population~$X_n$.
Conditionally on~$X_n$,
the random variables $Y_i, 1\leq i\leq m/2$, are i.i.d.
with values in $\{\,0,1,2\,\}$, yet
the computations are a bit lengthy and tedious because we have to
consider all the possible cases, depending on whether the parents
of
$X_{n+1}(2i-1), X_{n+1}(2i)$ 
belong or not to the progeny of the initial Master sequence.
So let us focus on one pair of chromosomes, for instance the first one
$(X_{n+1}(1), X_{n+1}(2))$. 
We have to estimate all the conditional probabilities
$$P\big(\text{there is } 0,1\text{ or }2 \text{ Master sequences in }
(X_{n+1}(1), X_{n+1}(2)) 
\,\big|\,X_n\big)\,.$$
To control these probabilities, we introduce
the time~$\tau_2$, when a mutant, not belonging to the progeny
of the initial Master sequence, has at least $\sqrt{\ell}$ ones.
We set
$$\displaylines{
\tau_2\,=\,\inf\,\Big\{\,n\geq 1: 
\begin{matrix}
\text{a chromosome of $X_n$ not in the progeny}\\
\text{
of the initial Master sequence has $\sqrt{\ell}$ ones}
\end{matrix}
\,\Big\}\,.}$$
Let $\lambda>0$ be such that 
$\pi/2\geq \exp(-\lambda)$.
We have then
$$(1-p_M)^\ell\,\geq\,\frac{\pi}{2(1-p_C)}\,\geq\,
\frac{\pi}{2}\,\geq\,
\exp(-\lambda)\,.$$
Notice that $\lambda$ depends only on $\pi$, and not on $\ell$ or $p_M$.
By lemma~\ref{binopoi}, the binomial law $\cB(\ell,p_M)$ is then
stochastically dominated by the Poisson law $\cP(\lambda)$. 
We will use repeatedly the bound on the tail of the Poisson law
given in lemma~\ref{poisstail}: 
$$\forall t\geq\lambda\qquad
P\Big(
\begin{matrix}
\text{a given chromosome undergoes at least }t\cr
\text{ mutations from one generation to the next}
\end{matrix}
\Big)\,\leq\,
\Big(\frac{\lambda e}{t}\Big)^t\,.$$
When using this bound, the value of~$t$ will be a function of~$\ell$.
We will always take $\ell$ large enough, so that the
value of $t$ will
be larger than $\lambda$.
We prove next a bound on~$\tau_2$.
\begin{lemma}\label{bounddau}
For $m\geq 2$ and for $\ell$ large enough, we have
$$P\Big(\tau_2\leq
\frac{1}
{5}
{\ln\ell}
\Big)\,\leq\,
1-\exp\Big(-m\exp\big(-\ell^{1/4}\big)\Big)\,.
$$
\end{lemma}
\begin{proof}
If $\tau_2<n$, then, before time $n$, a chromosome has been
created with at least $\sqrt{\ell}$ ones, and whose genealogy
does not contain the initial Master sequence.
We shall compute an upper bound on the number of ones
appearing in the genealogy of such a chromosome at generation~$n$.
Let us define
$D_n$ as the maximum number of ones in 
a chromosome
of the generation~$n$, which does not belong to the progeny
of the initial Master sequence.
These ones must have been created by mutation.
Let us consider 
a chromosome
of the generation~$n+1$, which does not belong to the progeny
of the initial Master sequence.
The number of ones in each of its two parents was at most~$D_n$.
After crossover between these two parents, the number of ones
was at most~$2D_n$.
After mutation, the number of ones
was at most
$$D_{n+1}\,\leq\, 2D_n+
\max\,\Big\{
\begin{matrix}
\text{number of mutations occurring on a }\cr
\text{ chromosome between generation $n$ and $n+1$}
\end{matrix}
\Big\}\,.$$
We first control the last term.
Let $n\geq 1$ and let us define the event
$\cE(n)$ by
$$\cE(n)=
\Big\{\,
\begin{matrix}
\text{until generation $n$, during the mutation process, the number
}\cr
\text{of mutations occurring on a given chromosome
is at most $\ell^{1/4}\!$}
\end{matrix}
\,\Big\}.$$
We have
$$\displaylines{
P\big(\cE(n)\big)
\,=\,\Big(1-
P\Big(
\begin{matrix}
\text{a given chromosome undergoes}\cr
\text{more than $\ell^{1/4}$
mutations}
\end{matrix}
\Big)
\Big)^{mn}
\,.
}$$
Using the bound given in lemma~\ref{poisstail}, we obtain that,
for $\ell^{1/4}>\lambda$, 
$$
P\big(\cE(n)\big)
\,\geq\,\Big(1-
\Big(\frac{\lambda e}{\ell^{1/4}}\Big)^{\ell^{1/4}}\Big)^{mn}\,,$$
whence, for $\ell$ large enough,
$$
P\big(\cE(n)\big)
\,\geq\,\exp\Big(-mn\exp\big(-\ell^{1/4}\big)\Big)\,.$$
Suppose that the event $\cE(n)$ occurs. We have then
$$\forall k\in\{\,0,\dots,n-1\,\}\qquad
D_{k+1}\,\leq\, 2D_k+\ell^{1/4}\,.$$
Dividing by $2^{k+1}$ and summing from $k=0$ to $n-1$, we get
$$D_n\,\leq\,2^n\sum_{k=0}^{n-1}
\frac{\ell^{1/4}}{2^{k+1}}
\,\leq\,
2^n
{\ell^{1/4}}
\,.$$
Therefore, if $2^n<
{\ell^{1/4}}$ and if the event $\cE(n)$ occurs, then
$\tau_2>n$. Taking $n=(\ln\ell)/5$, 
we obtain the estimate stated in the lemma.
\end{proof}

\noindent
We recall that
$$\tau_1\,=\,\inf\,\big\{\,n\geq 1: T_n>m^{1/4}\,\big\}\,.$$
We set also
$$\tau_0\,=\,\inf\,\big\{\,n\geq 1: N^*_n=0\,\big\}\,.$$
We shall compute a bound on $N^*_n$ until time 
$\tau=\min(\tau_0,\tau_1,\tau_2)$.
Our goal is to show that, for $m$ large enough, the process
$$(N^*_n
1_{\{\,\tau\geq n\,\}}
)_{n\in\mathbb N}$$
is stochastically dominated by
a subcritical
Galton--Watson process.
So let $n\geq 0$ and 
let us suppose that
$\tau> n$ and that we know the population~$X_n$.
We estimate the probability that
exactly 
one Master sequence
is present in 
$(X_{n+1}(1), X_{n+1}(2))$. 
We envisage different scenarios, depending on the number of descendants
of the initial Master sequence among the two parents of these chromosomes.
\medskip

\noindent
$\bullet$ First scenario. The two parents are descendants of the Master sequence.
The probability of selecting such two parents is bounded from above by
$$
\Big(\frac{2T_{n}}{2T_n+m-T_{n}}\Big)^2\,\leq\,
\Big(
\frac{2T_{n}}{m}\Big)^2
\,\leq\,
\frac{4}{m\sqrt{m}}
\,.$$
\medskip

\noindent
$\bullet$ Second scenario. Exactly one of the parents is a descendant of the Master sequence and a crossover has occurred. The total number of ones present in the parents
is at most $\ell+\sqrt{\ell}$.
After crossover,
the probability that one of the two resulting chromosomes 
has at least $\ell-\sqrt{\ell}$ ones is less than $4/\sqrt{\ell}$.
Indeed, this can happen only if, either on the left of
the cutting site, or on its right, there are at most $\sqrt{\ell}$ zeroes. 
The most favorable
situation is when all the ones are at the end or at the
beginning of the chromosome which is not a descendant of the Master sequence, 
in which case we have 
$2\sqrt{\ell}$ cutting sites which lead to the desired result.
Otherwise, both chromosomes after crossover have at least $\sqrt{\ell}$ zeroes,
and the probability to transform these zeroes into ones through mutations is
less than
$\smash{({\lambda e}/{\sqrt{\ell}})^{\sqrt{\ell}}}$.
We conclude that the probability of this scenario is bound from above by
$$\bigg(
\frac{4}{\sqrt{\ell}}\,+\,
2\Big(\frac{\lambda e}{\sqrt{\ell}}\Big)^{\sqrt{\ell}}\bigg)
\frac{2N^*_{n}}{2N^*_n+m-N^*_{n}}
\,.$$
\medskip

\noindent
$\bullet$ Third scenario. Exactly one of the parents is a descendant of the Master sequence and no crossover has occurred.
A Master sequence can be created from the chromosome not in the progeny of the
initial Master sequence, this would require $\ell-\sqrt{\ell}$ mutations,
and the corresponding probability is bounded from above by
$$\Big(\frac{\lambda e}{\ell-\sqrt{\ell}}\Big)^{\ell-\sqrt{\ell}}\,.$$
The other possibility is that
a Master sequence is obtained from
the chromosome belonging to the progeny of the
initial Master sequence. 
This chromosome was either a Master sequence, in which case the replication
has to be exact, or it was differing from the Master sequence, in which case
some mutations are required.
The corresponding probability is bounded from above by
$$2(1-p_C)\Big((1-p_M)^\ell+p_M\Big)
\frac{2N^*_{n}}{2N^*_n+m-N^*_{n}}\,.
$$
\medskip

\noindent
$\bullet$ Fourth scenario. None of the parents is a descendant of the Master sequence.
Until time $\tau_2$,
the chromosomes which are not descendants of the Master sequence have at most
$\sqrt{\ell}$ ones.
To create a Master sequence starting from two such parents require at least
$\ell-2\sqrt{\ell}$ mutations. The corresponding probability is bounded from above
by
$$2\Big(\frac{\lambda e}{\ell-2\sqrt{\ell}}\Big)^{\ell-2\sqrt{\ell}}\,.$$
Putting together the previous inequalities, we conclude that
$$\displaylines{
P\Big(
\begin{matrix}
\text{
there is exactly one Master sequence
}\cr
\text{
present in $(X_{n+1}(1), X_{n+1}(2))$ }
\end{matrix}
\,\Big|\,
\begin{matrix}
T_n\cr
N^*_n
\end{matrix}
\Big)\,\leq\,\hfill
\cr
\frac{4}{m\sqrt{m}}\,+\,
\frac{4}{\sqrt{\ell}}\,+\,
\bigg(
\frac{4}{\sqrt{\ell}}\,+\,
2\Big(\frac{\lambda e}{\sqrt{\ell}}\Big)^{\sqrt{\ell}}\bigg)
\frac{2N^*_{n}}{2N^*_n+m-N^*_{n}}
\hfill\cr
\hfill
\,+\,
2(1-p_C)\Big((1-p_M)^\ell+p_M\Big)
\frac{2N^*_{n}}{2N^*_n+m-N^*_{n}}+
2\Big(\frac{\lambda e}{\ell-2\sqrt{\ell}}\Big)^{\ell-2\sqrt{\ell}}
\,.
}$$
We rewrite the previous inequalities in the case $\ell=m$ and for $m$ large.
Since
$2(1-p_M)^m\,\geq\,\pi$,
then
$p_M\,\leq\,-\frac{1}{m}\ln(\pi/2)$.
Let $\varepsilon>0$ 
 be such that $\pi(1+5\varepsilon) <1$.
For $m$ large enough and $n<{\tau}$, we have 
$$\displaylines{
P\Big(
\begin{matrix}
\text{
there is exactly one Master sequence
}\cr
\text{
present in $(X_{n+1}(1), X_{n+1}(2))$ }
\end{matrix}
\,\Big|\,
\begin{matrix}
T_n\cr
N^*_n
\end{matrix}
\Big)\,\leq\,
\frac{2} {m}
 \pi(1+\varepsilon) 
N_{n}^*
\,.
}$$
Similar computations yield that 
there exists
a positive constant $c$ such that,
for $m$ large enough and $n<{\tau}$, 
$$\displaylines{
P\Big(
\begin{matrix}
\text{both $X_{n+1}(1), X_{n+1}(2)$}\cr
\text{are Master sequences}
\end{matrix}
\,\Big|\,
\begin{matrix}
T_n\cr
N^*_n
\end{matrix}
\Big)\,\leq\,
\frac{c}{m^{3/2}}
\,.}$$
Coming back to the initial equality for~$N_{n+1}^*$, 
we conclude that, for $m$ large enough,
the law of
$N^*_{n+1}
1_{\{\,\tau\geq n+1\,\}}$ is stochastically dominated by the sum of two
independent binomial random variables as follows:
$$
N^*_{n+1}
1_{\{\,\tau\geq n+1\,\}}\,\preceq\,
\cB\Big(
\frac{m}{2},
\frac{2
}{m}
 \pi(1+2\varepsilon) 
N^*_{n}
1_{\{\,\tau\geq n\,\}}
\Big)
+
2\cB\Big(
\frac{m}{2},
\frac{c}{m^{3/2}}
\Big)\,.
$$
For $m$ large, these two binomial laws
are in turn stochastically
dominated by two Poisson laws.
More precisely, 
for $m$ large enough,
$$\displaylines{
\Big(1-
\frac{2}{m}
 \pi(1+2\varepsilon) 
N^*_{n}
1_{\{\,\tau\geq n\,\}}
\Big)^{m/2}\geq
\exp\Big(-
 \pi(1+3\varepsilon) 
N^*_{n}
1_{\{\,\tau\geq n\,\}}
\Big)
\,,\cr
\big(1-
{c}{m^{-3/2}}
\big)^{m/2}\geq\exp(-\varepsilon)\,.}$$
Lemma~\ref{binopoi} yields then that
$$
N^*_{n+1}
1_{\{\,\tau\geq n+1\,\}}\,\preceq\,
\cP\big(
 \pi(1+3\varepsilon) 
N^*_{n}
1_{\{\,\tau\geq n\,\}}
\big)
+
2\cP(
\varepsilon
)\,.
$$
The point is that we have got rid of the variable $m$ in the upper bound, so we are
now in position to compare 
$N^*_{n}
1_{\{\,\tau\geq n\,\}}$
with a Galton--Watson process.
Let 
$(Y'_n)_{n\geq 1}$
be a sequence of i.i.d. random variables with law
$\cP(
 \pi(1+3\varepsilon) 
)$, 
let 
$(Y''_n)_{n\geq 1}$
be a sequence of i.i.d. random variables with law
$\cP(\varepsilon)$, both sequences being independent.
The previous stochastic inequality can be rewritten as
$$
N^*_{n+1}
1_{\{\,\tau\geq n+1\,\}}\,\preceq\,
\Big(
\sum_{k\geq 1}^{
N^*_{n}
1_{\{\,\tau\geq n\,\}}}
Y'_k
\Big)
+2Y''_1
\,.
$$
This implies further that
$$
N^*_{n+1}
1_{\{\,\tau\geq n+1\,\}}\,\preceq\,
\sum_{k\geq 1}^{
N^*_{n}
1_{\{\,\tau\geq n\,\}}}
\big(
Y'_k+2Y''_k
\big)\,.
\qquad(\star)
$$
Let $\nu^*$ be the law of 
$Y'_1+2Y''_1$  and let
$(Z^*_n)_{n\geq 0}$ be a Galton--Watson process starting from $Z_0=1$
with reproduction law~$\nu^*$. We prove finally that,
for $m$ large enough, 
$$\forall n\geq 0\qquad
N^*_{n}
1_{\{\,\tau\geq n\,\}}\,\preceq\,
Z^*_n\,.$$
We suppose that $m$ is large enough so that the stochastic 
inequality~$(\star)$ holds and we proceed by induction on~$n$.
For $n=0$, we have
$$
N^*_{0}
1_{\{\,\tau\geq 0\,\}}\,=\,1\,\leq\,
Z^*_0\,=\,1\,.$$
Let $n\geq 0$ and 
suppose that the inequality holds at rank $n$.
Inequality~$(\star)$ yields
$$
N^*_{n+1}
1_{\{\,\tau\geq n+1\,\}}\,\preceq\,
\sum_{k\geq 1}^{
N^*_{n}
1_{\{\,\tau\geq n\,\}}}
\kern-7pt
\big(
Y'_k+2Y''_k
\big)
\,\preceq\,
\sum_{k\geq 1}^{
Z^*_{n}
}
\big(
Y'_k+2Y''_k
\big)
\,=\,
Z^*_{n+1}
\,.
$$
Thus the inequality holds at rank~$n$ and the induction is
completed.
Moreover we have
$$E(\nu^*)\,=\,E
\big(
Y'_1+2Y''_1
\big)\,=\,
 \pi(1+5\varepsilon) \,<\,1\,.
$$
Thus the 
Galton--Watson process 
$(Z^*_n)_{n\geq 0}$ is subcritical.

We complete now the proof of
theorem~\ref{thspl}.
Let $\kappa,c_1>0$ be 
constants associated to the Galton--Watson process
$(Z_n)_{n\geq 0}$ 
as in proposition~\ref{boundtau}. 
We suppose that $\kappa<1/5$, so that we can use the estimate
of
lemma~\ref{bounddau}. 
Let $c>0$ be a constant associated to the subcritical
Galton--Watson process
$(Z^*_n)_{n\geq 0}$ 
as in lemma~\ref{subgw}.
We have then
$$\displaylines{
P\big(\tau_0> \kappa\ln m\big)\,\leq\,
\hfill\cr
P\big(
\tau_0> \kappa\ln m,\,
\tau< \kappa\ln m\big)\,+\,
P\big(
N_{\lfloor\kappa\ln m\rfloor}^*>0,\,
\tau\geq \kappa\ln m\big)\cr
\,\leq\,
P\big(\tau_1<\kappa\ln m\big)\,+\,
P\big(\tau_2<\kappa\ln m\big)\,+\,
P\big(Z_{\lfloor\kappa\ln m\rfloor}^*>0\big)\cr
\,\leq\,
\frac{1}{m^{c_1}}
\,+\,
1-\exp(-m\exp\big(-m^{1/4}\big)\Big)
\,+\,
\,\exp(-c^* {\lfloor\kappa\ln m\rfloor})\,.
}$$
This inequality yields the estimate stated in theorem~\ref{thspl}.

\section{The quasispecies regime}\label{secqua}

In this section, we will prove theorem~\ref{thqsr}.
We start the genetic algorithm with an initial population whose maximal fitness
is equal to $f_0^*$ and whose mean fitness is equal to
${\overline{f_0}}$.
For $x=(x(1),\dots,x(m))$ a population,
we define $N(x,f_0^*)$ as the
number of chromosomes in $x$ whose fitness is larger than or equal to~$f_0^*$:
$$N(x,f_0^*)\,=\,\card\,\{\,i\in\um:f(x(i))\geq f_0^*\,\}\,.$$
We denote by $X_n$ the population at generation~$n$ and by 
$X_n(1),\dots,X_n(m)$ the $m$ chromosomes of~$X_n$.
We define a stopping time $\overline{\tau}$ by
$$\overline{\tau}\,=\,\inf\,\Big\{\,n\geq 1: 
\frac{1}{m}\Big(f(X_n(1))+\cdots+f(X_n(m))\Big)\geq
\sqrt{\pi}\,
\overline{f_0}
\,\Big\}\,.$$
Our goal is to control the time $\overline{\tau}$, more
precisely we would like to prove that
$\overline{\tau}$ is less than $\kappa\ln m$ with high probability.
Unfortunately, the process
$\smash{\big(
(N(X_n,f_0^*) \big)_{n\geq 0}}$ is very complicated, it is not even
a Markov process. Our strategy is to construct 
an auxiliary Markov chain
which is considerably simpler and which bounds
$\smash{\big(
(N(X_n,f_0^*) \big)_{n\geq 0}}$ from below until time
$\overline{\tau}$.
The production of chromosomes with fitness larger than or equal to $f_0^*$ from one generation
to the next can be decomposed into two distinct mechanisms:
\smallskip

\noindent
$\bullet$ chromosomes which are an exact copy of one of their parents;
\smallskip

\noindent
$\bullet$ chromosomes which have undergone mutation or crossover events.
\smallskip

\noindent
We will bound from below the process 
$\smash{\big( (N(X_n,f_0^*) \big)_{n\geq 0}}$
by neglecting the second mechanism. The key point is that the law of the number
of chromosomes created in generation $n+1$
through the first mechanism depends only on the
value
$N(X_n,f_0^*)$
and not on the detailed composition of the population at time~$n$.
Therefore we are able to obtain a lower process which is a Markov chain. We denote this process
by
$\smash{(N_n)_{n\geq 0}}$. We proceed next to its precise definition.
Suppose that in the generation~$n$, we have $i$ chromosomes of fitness larger than
or equal to $f_0^*$, and that the mean fitness is still below
$\sqrt{\pi}\overline{f_0}$, that is, we condition on the event 
$N(X_n,f_0^*)=i$ and
$\overline{\tau}>n$.
Let us look at the first pair of chromosomes of generation $n+1$. 
The probabibility to select from the generation $n$ 
a chromosome of fitness larger than or equal to
$f_0^*$ is at least
${if_0^*}/\big({m\sqrt{\pi}\,
\overline{f_0}})$.
The probability that no crossover 
has occurred is $1-p_C$.
The probability that no mutation has occurred on a given chromosome
is $(1-p_M)^\ell$.
Thus the probability that the first chromosome of the generation~$n+1$
is an 
exact copy of a chromosome of generation~$n$ having fitness larger than
or equal to $f_0^*$ is at least
$$\frac{if_0^*}{m\sqrt{\pi}\,
\overline{f_0}}
(1-p_C)
(1-p_M)^\ell\,.$$
However the crossover creates correlations between adjacent chromosomes, so the 
distribution of
$N_{n+1}$ cannot be taken simply as a binomial law.
Conditionally on the event that $N(X_n,f^*_0)=i$ 
and $\overline{\tau}>n$,
a correct lower bound on 
$N(X_{n+1},f_0^*)$ 
is given by the sum
$$\sum_{k=0}^{m/2}
Z_k\big(Y_{2k-1}+Y_{2k}\big)\,,$$
where $Z_1,\dots,Z_{m/2}$ are Bernoulli with parameter $1-p_C$,
and
$Y_1,\dots,Y_m$ are Bernoulli with parameter
$$\varepsilon_m(i)\,=\,
\frac{if_0^*}{m\sqrt{\pi}\,
\overline{f_0}}(1-p_M)^\ell\,$$
and they are all independent.
The variable $Z_k$ is $1$ if there was no crossover between the chromosomes
of the $k$--th pair and $0$ otherwise. 
The variable $Y_k$ is $1$ if the $k$--th chromosome selected has fitness larger than or equal to
$f_0^*$ and it is not affected by any mutation.
We obtain that, for 
$j\in\zm$,
$$P\Big(
N(X_{n+1},f_0^*) \geq j\,\Big|\,
\begin{matrix}
N(X_{n},f_0^*) =i\\
\overline{\tau}>n
\end{matrix}
\Big)\,\geq\,
P\Big(
\sum_{k=0}^{m/2}
Z_k\big(Y_{2k-1}+Y_{2k}\big)\,\geq\,j\,\Big)\,.$$
We compute the righthand side and we are led to define the transition
matrix
of the Markov chain
$\smash{\big(N_n\big)_{n\geq 0}}$ by setting,
for $i,j\in\zm$,
$$\displaylines{
P\big(
N_{n+1}=j
\,|\,
N_{n}=i
\,\big)
\,=\,
\hfill\cr
\sum_{b=0}^{m/2}
\begin{pmatrix}
{m/2}\\
{b}
\end{pmatrix}
(1-p_C)^b
p_C^{m/2-b}
\begin{pmatrix}
{2b}\\
{j}
\end{pmatrix}
\varepsilon_m(i)^j
(1-\varepsilon_m(i))^{2b-j}\,.
}$$
The above inequality can then be rewritten as:
for $i,j\in\zm$,
$$P\Big(
N(X_{n+1},f_0^*) \geq j\,\Big|\,
\begin{matrix}
N(X_{n},f_0^*) =i\\
\overline{\tau}>n
\end{matrix}
\Big)\,\geq\,
P\big(
N_{n+1}\geq j
\,|\,
N_{n}=i
\,\big)
\,.$$
From lemma~\ref{stoint}, 
this implies furthermore that, for any non--decreasing function 
$\phi:\N\to\R$,
for $i\in\zm$,
$$E\Big(
\phi\big(N(X_{n+1},f_0^*)\big)
\,\Big|\,
\begin{matrix}
N(X_{n},f_0^*) =i\\
\overline{\tau}>n
\end{matrix}
\Big)\,\geq\,
E\Big(
\phi\big(
N_{n+1}\big)
\,|\,
N_{n}=i
\,\Big)
\,.\qquad(\diamond)$$
Let us focus a bit on the 
the Markov chain 
$\smash{\big(N_n\big)_{n\geq 0}}$. Its state space is $\zm$.
The null state is an absorbing state because we neglect the
mutations for producing chromosomes of fitness at least $f_0^*$.
A key point to exploit inequality~$(\diamond)$
is the following result.
\begin{proposition}\label{monot}
The Markov chain 
$\smash{\big(N_n\big)_{n\geq 0}}$
is monotone.
\end{proposition}
\begin{proof}
The definition of monotone Markov chain
is recalled in appendix
(see definition~\ref{defmo}). The easiest way to prove the
monotonicity is to build an adequate coupling.
For $n\in\N$ and $k\leq m/2$, let
$Z_k^n$ be a Bernoulli random variable with parameter $1-p_C$
and
$U_{2k-1}^n,U_{2k}^n$ be two random variables whose distribution
is uniform over $[0,1]$.
We suppose that all the above random variables are independent.
For $i\in\zm$, we define $N_0^i=i$ and
$$\forall n\geq 0\qquad
N_{n+1}^i\,=\,
\sum_{k=0}^{m/2}
Z_k^n\big(
1_{\{U_{2k-1}^n<\varepsilon_m(N_n)\}}
+
1_{\{U_{2k}^n<\varepsilon_m(N_n)\}}
\big)\,.$$
This way all the chains
$\smash{\big(N_n\big)_{n\geq 0}},\,i\in\zm$,
are coupled and a straightforward induction yields that
$$\forall i\leq j\quad\forall n\in\N\qquad N_n^i\,\leq\,N_n^j\,.$$
This yields the desired conclusion.
\end{proof}

\noindent
We are interested in the process
$\smash{\big( (N(X_n,f_0^*) \big)_{n\geq 0}}$
until time 
$\overline{\tau}$. 
In order to prove a convenient stochastic inequality, we will
work with
the process
$\smash{\big( N^*_n\big)_{n\geq 0}}$
defined
by
$$\forall n\geq 0\qquad
N^*_n 
\,=\,
\begin{cases}
{N}(X_n,f_0^*) 
&\text{if }\overline{\tau}>n\\
m
&\text{if }\overline{\tau}\leq n\\
\end{cases}
$$
\begin{proposition}\label{stod}
We suppose that the Markov chain
$(N_n)_{n\in\mathbb N}$ starts from $N_0=1$. 
For any $n\geq 0$, we have the stochastic inequality
$$\big(
N^*_0,\dots
N^*_n
\big)\,\succeq\,
\big(
{N}_0,\dots
{N}_n
\big)\,.
$$
\end{proposition}
For the above statement, we work with the product order on $\N^{n+1}$:
$$(i_0,\dots,i_n)\,\geq (j_0,\dots,j_n)\quad\Longleftrightarrow\quad
i_0\geq j_0,\dots,
i_n\geq j_n\,.$$
The stochastic domination inequality stated in proposition~\ref{stod} means
that: for any non--decreasing function~$\phi:\N^{n+1}\to\R^+$, we have
$$E\Big(\phi\big(
N^*_0,\dots
N^*_n
\big)\Big)\,\geq\,
E\Big(\phi\big(
{N}_0,\dots
{N}_n
\big)\Big)\,.
$$
\begin{proof}
We proceed by induction on~$n$.
For $n=0$, we have
$$N^*_0\,=\,
{N}(X_0,f_0^*)\,\geq\,1\,=\,N_0\,.$$
Suppose that the result has been proved until rank~$n$ for some $n\geq 0$. 
Let
$\phi:\N^{n+2}\to\R^+$ be a
non--decreasing function. We write
$$\displaylines{
E\Big(\phi\big(
N^*_0,\dots,
N^*_{n+1}
\big)\Big)\,=\,
\sum_{0\leq i_0,\dots,i_n\leq m}
P
\Big(
N^*_0=i_0,\dots,
N^*_n=i_n
\Big)\hfill
\cr
\hfill\times\,
E\Big(\phi\big(
N^*_0,\dots,
N^*_{n+1}
\big)
\,\Big|\,
N^*_0=i_0,\dots,
N^*_n=i_n
\Big)\,.
}$$
Let 
$i_0,\dots,i_n$ be fixed. 
Suppose first that $i_n<m$. The event 
$\{\,N^*_n=i_n\,\}$ implies that 
$\overline{\tau}>n$ and 
$N^*_n=N(X_n,f_0^*)$. 
The map
$$i\in\zm\mapsto
\phi\big(
i_0,\dots,i_n,i
\big)$$
is non--decreasing.
Using the stochastic inequality~$(\diamond)$,
we obtain
$$\displaylines{
E\Big(\phi\big(
N^*_0,\dots,
N^*_{n+1}
\big)
\,\Big|\,
N^*_0=i_0,\dots,
N^*_n=i_n
\Big)\,=\,\hfill\cr
E\Big(\phi\big(
i_0,\dots,i_n,
N(X_{n+1},f_0^*)
\big)
\,\Big|\,
N(X_{0},f_0^*)
=i_0,\dots,
N(X_{n},f_0^*)
=i_n
\Big)\cr
\,\geq\,
E\Big(\phi\big(
i_0,\dots,i_n,
N_{n+1}
\big)
\,\Big|\,
N_{n}
=i_n
\Big)
\,.
}$$
Let us define a function
$\psi:\N^{n+1}\to\R^+$ by setting
$$\psi\big(
i_0,\dots,i_n)\,=\,
E\Big(\phi\big(
i_0,\dots,i_n,
N_{n+1}
\big)
\,\Big|\,
N_{n}
=i_n
\Big)
\,.$$
If $i_n=m$, then
we have also
$$\displaylines{
E\Big(\phi\big(
N^*_0,\dots,
N^*_{n+1}
\big)
\,\Big|\,
N^*_0=i_0,\dots,
N^*_n=i_n
\Big)\,=\,
\phi\big(
i_0,\dots,i_{n-1},m,m
\big)
\cr\hfill
\,\geq\,
\psi\big(
i_0,\dots,i_n)
\,.
}$$
From the previous inequalities, we conclude that
$$\displaylines{
E\Big(\phi\big(
N^*_0,\dots,
N^*_{n+1}
\big)\Big)\,\geq\,
\hfill\cr
\sum_{0\leq i_0,\dots,i_n\leq m}
P
\Big(
N^*_0=i_0,\dots,
N^*_n=i_n
\Big)
\psi\big(
i_0,\dots,i_n)
\,=\,
E\Big(\psi\big(
N^*_0,\dots
N^*_{n}
\big)\Big)\,.
}$$
Since the function $\phi$ is non--decreasing on $\N^{n+2}$ and since
the 
Markov chain 
$\smash{(N_n)_{n\geq 0}}$
is monotone (by proposition~\ref{monot}), then
the function $\psi$ is
also non--decreasing on $\N^{n+1}$.
Now, the induction hypothesis yields that
$$E\Big(\psi\big(
N^*_0,\dots,
N^*_{n}
\big)\Big)\,\geq\,
E\Big(\psi\big(
N_0,\dots,
N_{n}
\big)\Big)\,=\,
E\Big(\phi\big(
N_0,\dots,
N_{n},
N_{n+1}
\big)\Big)
$$
and the induction step is completed.
\end{proof}

\noindent
If $N(X_n,f_0^*)>
{m}/{\sqrt\pi}$,
then necessarily
$$\frac{1}{m}\Big(f(X_n(1))+\cdots+f(X_n(m))\Big)>
\frac{1}{\sqrt\pi} 
f_0^*\,\geq\,
\sqrt{\pi}\,
\overline{f_0}
$$
and thus 
${\overline{\tau}}< n$.
The above coupling inequality implies therefore that
$$\displaylines{
P\big({\overline{\tau}}< n\big)\,\geq\,
P\big(\exists k\leq n\quad 
N(X_k,f_0^*)> {m}/{\sqrt\pi}
\big)\cr\hfil
\,\geq\,
P\big(\exists k\leq n\quad N_k> m/\sqrt{\pi}\big)\,.}$$
We study next the dynamics of the
Markov chain $\smash{(N_n)_{n\geq 0}}$ on $\zm$.
Our goal is to prove that, for some $\kappa>0$,
with a probability larger than a constant independent of~$m$,
this Markov chain will reach a value strictly
larger than~$m/\sqrt{\pi}$ before time $\kappa\ln m$.
Let us explain briefly the heuristics for this result.
The transition mechanism of the chain is built with the help of i.i.d. Bernoulli random
variables, some of parameter $1-p_C$ and some of parameter $\epsilon_m(i)$, $i\in\zm$.
The typical number of pairs of chromosomes with no crossover from one generation
to another is
$(1-p_C)m/2$ and we can control accurately the deviations from this typical value.
For $i$ small compared to~$m$, the parameter
$\epsilon_m(i)$ is of order $\text{cte}\times i/m$, thus,
conditionally on the event that $N_n=i$,
the distribution of $N_{n+1}$ is 
roughly the binomial law of parameters $m(1-p_C)$ and $\text{cte}\times i/m$.
In this regime, it can
be approximated adequately by a Poisson law of parameter
$$ m(1-p_C)\epsilon_m(i)\,\sim\,i\sqrt{\pi}\,.$$
We conclude that, as long as $N_n$ is small compared to~$m$, we have
$$ E(N_{n+1})\,\sim\,\sqrt{\pi}E(N_n)\,.$$
In the next proposition, we derive a rigorous estimate,
which shows indeed that the Markov
chain $\smash{(N_n)_{n\geq 0}}$ is likely to grow geometrically until
a value larger than~$m/\sqrt{\pi}$.
The proof is elementary, in the sense that it relies essentially on two classical
exponential inequalities (which are recalled in the appendix). This proof
is an adaptation of the proof of proposition~$6.7$ in \cite{sga}.
In proposition~\ref{sfj}, we shall then bound from below the probability
of hitting a value larger than~$m/\sqrt{\pi}$ before time $\kappa\ln m$
and this will conclude the proof of theorem~\ref{thqsr}.
\begin{proposition}\label{crqa}
Let $\pi>1$ be fixed.
There exist 
$\rho>1$, $c_0>0$, $m_0\geq 1$,
which depend
on $\pi$ and the ratio ${f_0^*}/{\overline{f}_0}$
only, such that:
for any set of parameters
$\ell,p_C,p_M$ satisfying
$\pi
\,=\,\big({f_0^*}/{\overline{f_0}}\big)
(1-p_C)(1-p_M)^\ell$,
we have
$$\displaylines{
\forall m\geq m_0\quad
\forall i\leq m/\sqrt{\pi}\quad
P\big(\,
N_{n+1}\leq\rho i\,\big|\,N_n=i\,\big)
\,\leq\,\exp(-c_0i)\,.}$$
\end{proposition}
\begin{proof}
We recall that,
conditionally on $N_n=i$, the law of $N_{n+1}$ is the same
as the law of the random variable
$$\sum_{k=1}^{2B_n}Y_k^i\,,$$
where 
$B_n$ is distributed according to the binomial law $\cB(m/2,1-p_C)$,
the variables
$Y_k^i$, $k\in\N$, $i\in\um$, are Bernoulli random variables
with parameter
$\varepsilon_m(i)$, and all these random variables are independent.
Let $\varepsilon>0$ be such that $\sqrt{\pi}(1-2\varepsilon)>1$ and let
$$l(m,\varepsilon)\,=\,
\Big\lfloor\frac{m}{2}(1-p_C)(1-\varepsilon)\Big\rfloor +1+
\frac{m}{4}(1-p_C)\varepsilon
\,.$$
For $m$ large enough, we have
$$l(m,\varepsilon)\,<\,
\frac{m}{2}(1-p_C)\big(1-
\frac{\varepsilon}{2}\big)+1
\,<\,
\frac{m}{2}(1-p_C)
\,.$$
Let $\rho$ be such that
$1<\rho<\sqrt{\pi}(1-2\varepsilon)$. 
We have
\begin{align*}
P\big(\,
N_{n+1}<\rho i\,\big|\,N_n=i\,\big)
&\,=\,
P\Big(\sum_{k=1}^{2B_n}Y_k^i<\rho i\Big)\cr
&\,\leq\,
P\big(B_n\leq l(m,\varepsilon)\big)+
P\Big(\sum_{k=1}^{2l(m,\varepsilon)}Y_k^i<\rho i\Big)
\,.
\end{align*}
We control the first probability with the help of Hoeffding's inequality
(see the appendix).
The expected value of $B_n$ is
${m}(1-p_C)/2> l(m,\varepsilon)$, thus
$$P\big(B_n\leq l(m,\varepsilon)\big)
\,\leq\,
\exp\Big(
-\frac{2}{m}\Big(
\frac{m}{2}(1-p_C)-l(m,\varepsilon)\Big)^2\Big)
\,.
$$
Recall that $1-p_C>
{\overline{f_0}}/ {f_0^*}$.
For $m$ large enough, we have
$$\frac{m}{2}(1-p_C)-l(m,\varepsilon)\,\geq\,
\frac{m}{2}(1-p_C)\frac{\varepsilon}{2}-1\,
\,\geq\,\frac{m\varepsilon}{4}
\frac{{\overline{f_0}}}{{f_0^*}}
-1\,
\geq\,\frac{m\varepsilon}{8}
\frac{{\overline{f_0}}}{{f_0^*}}
\,.$$
It follows that,
for $m$ large enough, 
$$
P\big(B_n\leq l(m,\varepsilon)\big)
\,\leq\,
\exp\Big(
-\frac{m}{32}
\Big(\frac{
\varepsilon
{\overline{f_0}}}{{f_0^*}}
\Big)^2\Big)
\,.
$$
Let us try to apply also
Hoeffding's inequality to control the second probability.
We get
$$
P\Big(\sum_{k=1}^{2l(m,\varepsilon)}Y_k^i<\rho i\Big)
\,\leq\,\exp\Big(
-\frac{1}
{l(m,\varepsilon)}
\Big(
2{l(m,\varepsilon)}
\varepsilon_m(i)
-\rho i
\Big)^2
\Big)
\,.
$$
Now
$$2{l(m,\varepsilon)}
\varepsilon_m(i)
\,\geq\,
2\frac{m}{2}(1-p_C)(1-{\varepsilon})
\frac{if_0^*}{m\sqrt{\pi}\,
\overline{f_0}}(1-p_M)^\ell
\,=\,
(1-\varepsilon)
{i\sqrt{\pi}}
\,,$$
whence, using the hypothesis on~$\rho$,
$$
P\Big(\sum_{k=1}^{2l(m,\varepsilon)}Y_k^i<\rho i\Big)
\,\leq\,
\exp\Big(
-\frac{
 {\pi} \varepsilon^2i^2 } {m}
\Big)\,.$$
This inequality becomes useful only when $i$ of order $\delta m$ for some
$\delta>0$.
For smaller values of $i$, we must proceed differently in order to control this
probability. Thus we decompose
the sum into~$i$ blocks and we use the
Chebyshev exponential inequality. Each block
 follows a binomial law, and we bound the
Cram\'er transform of each block by
the Cram\'er transform of a Poisson law having
the same mean.
More precisely, we choose for the block size
$$
b\,=\,\Big\lfloor\frac{\displaystyle 2l(m,\varepsilon)
-\frac{m}{4}(1-p_C)\varepsilon
}{i}+1\Big\rfloor\,,$$
and we define
the sum associated to each block of size~$b$:
$$\forall j\in\{\,1,\dots,i\,\}\qquad
Y_j'\,=\,
\sum_{k=b(j-1)+1}^{bj}
Y_k^i\,.$$
Notice that 
$Y'_1$ follows the binomial law 
with parameters $b$ and $\varepsilon_m(i)$.
We will next estimate from below
the product
$b\varepsilon_m(i)$. 
By the choice of~$b$ and $l$, we have
$$\displaylines{
b\,\geq\,\frac{1}{i}\Big(2l(m,\varepsilon)
-\frac{m}{4}(1-p_C)\varepsilon
\Big)\,,\qquad
l(m,\varepsilon)\,\geq\,
\frac{m}{2}(1-p_C)
\Big(1-\frac{\varepsilon}{2}\Big)\,,
}$$
whence
$$b\,\geq\,\frac{m}{i}
(1-p_C)
\big(1-{\varepsilon}
\big)\,$$
and
$$
E(Y'_1)\,=\,
b\varepsilon_m(i)
\,\geq\,
\sqrt{\pi}
(1-\varepsilon)
\,>\,\rho
\,.
$$
Let $\delta_0>0$ be such that
$\delta_0\,<\,
(1-p_C)\varepsilon/4$.
Suppose that $i\leq \delta_0 m$.
We have also that
\begin{align*}
bi\,&\leq\,
2l(m,\varepsilon)
-\frac{m}{4}(1-p_C)\varepsilon+i
\cr
&\,\leq\,
2l(m,\varepsilon)
-\frac{m}{4}(1-p_C)\varepsilon+\delta_0 m
\,\leq\,
2l(m,\varepsilon)\,.
\end{align*}
Using the Chebyshev exponential inequality
(see the appendix),
we have then 
$$\displaylines{
P\Big(\sum_{k=1}^{2l(m,\varepsilon)}Y_k^i\leq\rho i\Big)
\,\leq\,
P\Big(\sum_{k=1}^{bi}Y_k^i\leq\rho i\Big)\hfill
\cr
\,\leq\,
P\Big(
\sum_{j=1}^{i}
Y_j' \leq\rho i\Big)
\,\leq\,
P\Big(
\sum_{j=1}^{i}
-Y_j' \geq-\rho i\Big)\,\leq\,
\exp\Big(-i\Lambda^*_{-Y'_1}(-\rho)\Big)
\,,
}$$
where
$\smash{\Lambda^*_{-Y'_1}}$ is the Cram\'er transform of $-Y'_1$.
Let $Y''_1$ be a random variable following the Poisson law of parameter
$b\varepsilon_m(i)$.
We shall use the following lemma to compare
the Cram\'er transforms 
of $-Y'_1$ and
$-Y''_1$.
By lemma~\ref{bnpp}, we have
$$\Lambda^*_{-Y'_1}(-\rho)\,\geq\,
\Lambda^*_{-Y''_1}(-\rho)\,=\,
\rho\ln\Big(\frac{\rho}{b\varepsilon_m(i)}\Big)-\rho+
{b\varepsilon_m(i)}\,.$$
The map
$$\lambda\mapsto
\rho\ln\Big(\frac{\rho}{\lambda}\Big)-\rho+
{\lambda}\,$$
is non--decreasing on $[\rho,+\infty[$ and
$b\varepsilon_m(i)\geq
\sqrt{\pi}(1-\varepsilon)$, thus
$$\Lambda^*_{-Y''_1}(-\rho)\,\geq\,
\rho\ln\Big(\frac{\rho}{
\sqrt{\pi}(1-\varepsilon)
}\Big)-\rho+
{\sqrt{\pi}(1-\varepsilon)}
\,.$$
Let us denote by $c_0$ the righthand quantity. Then $c_0$
is positive and it depends only on $\rho,\pi$,
${f_0^*}/{\overline{f}_0}$
and $\varepsilon$.
Finally, we have
for $m$ large enough,
$$\forall i\in\big\{\,1,\dots,\lfloor\delta_0 m\rfloor\,\big\}\qquad
P\Big(\sum_{k=1}^{2l(m,\varepsilon)}Y_k^i\leq\rho i\Big)
\,\leq\,\exp(-c_0i)\,,$$
whence
$$
P\big(\,
N_{n+1}\leq\rho i\,\big|\,N_n=i\,\big)
\,\leq\,
\exp\Big( -\frac{m}{32}
\Big(\frac{ \varepsilon {\overline{f_0}}}{{f_0^*}}
\Big)^2\Big)
\,+\,\exp(-c_0i)
\,.
$$
For $i$ such that $\delta_0 m\leq i < m/\sqrt{\pi}$, we had obtained
$$P\big(\,
N_{n+1}\leq\rho i\,\big|\,N_n=i\,\big)
\,\leq\,
\exp\Big( -\frac{m}{32}
\Big(\frac{ \varepsilon {\overline{f_0}}}{{f_0^*}}
\Big)^2\Big)
+
\exp\big(
-
 {\pi} \varepsilon^2\delta_0^2  {m}
\big)\,.$$
Let $\eta\in]0,1[$ be small enough so that 
$\eta c_0\leq \pi\varepsilon^2\delta_0^2$
and, for $m$ large enough,
$$
\exp\Big( -\frac{m}{32}
\Big(\frac{ \varepsilon {\overline{f_0}}}{{f_0^*}}
\Big)^2\Big)
\,\leq\,
\exp\Big( -\eta\frac{mc_0}{2}\Big)
\Big(1-
\exp\Big( -\eta\frac{c_0}{2}\Big)\Big)\,.$$
For $m$ large enough and
$i\in\big\{\,1,\dots,\lfloor\delta_0 m\rfloor\,\big\}$,
we have
\begin{multline*}
P\big(\,
N_{n+1}\leq\rho i\,\big|\,N_n=i\,\big)\cr
\,\leq\,
\exp\Big( -\eta\frac{ic_0}{2}\Big)
\Big(1-
\exp\Big( -\eta\frac{c_0}{2}\Big)\Big)
+\exp\big( -\eta{ic_0}\big)\cr
\,\leq\,
\exp\big( -\eta\frac{ic_0}{2}\big)\,.
\end{multline*}
For $m$ large enough and
$\delta_0m\leq i< m/\sqrt{\pi}$,
we have also
\begin{multline*}
P\big(\,
N_{n+1}\leq\rho i\,\big|\,N_n=i\,\big)\cr
\,\leq\,
\exp\Big( -\eta\frac{mc_0}{2}\Big)
\Big(1-
\exp\Big( -\eta\frac{c_0}{2}\Big)\Big)
+\exp\big( -\eta{mc_0}\big)\cr
\,\leq\,
\exp\big( -\eta\frac{ic_0}{2}\big)\,.
\end{multline*}
These inequalities yield the claim of the proposition.
\end{proof}

\noindent
We define
$$
\tau^*\,=\,\inf\,\{\,n\geq 0: N_n\geq m/\sqrt{\pi}\,\}\,.$$
\begin{proposition}\label{sfj}
Let $\pi>1$ be fixed.
There exist $\kappa>0$ and $p^*>0$ which depend 
on $\pi$ and the ratio ${f_0^*}/{\overline{f}_0}$ only
such that
$$
\forall m\geq 1\qquad
P\big(\tau^*\leq\kappa \ln m
\,|\,N_0=1\big)\,\geq\,
p^*\,.$$
\end{proposition}
\begin{proof}
Let us
define
$$
\tau_0\,=\,\inf\,\big\{\,n\geq 1: N_n=0\,\big\}\,.
$$
Recall that $0$ is an absorbing state. Thus, if the hitting time of $m$ is finite,
then necessarily, it is smaller than the 
hitting time of $0$. It follows that
$$
P\big(\tau^*\leq\kappa \ln m
\,|\,N_0=1\big)\,=\,
P\big(\tau^*\leq\kappa \ln m,\,\tau^*<\tau_0
\,|\,N_0=1\big)\,.$$
It is annoying to work with a Markov chain which has an absorbing state, so we first
get rid of this problem.
We consider the modified
Markov chain
$(\smash{\widetilde{N}}_n)_{n\geq 0}$ which has the same transition
probabilities as
$(N_n)_{n\geq 0}$, except that we set the transition probability from $0$
to $1$ to be $1$. 
The event we wish to estimate 
has the same probability for both processes, because
they have the same dynamics outside of~$0$.
So, from now onwards, we work with the Markov chain
$(\smash{\widetilde{N}}_n)_{n\geq 0}$, which is irreducible.
Let $\rho>1$, $c_0>0$, $m_0\geq 1$
be as given in
proposition~\ref{crqa}.
For $k\geq 0$,
let $T_k$ be the first time the process
$(\smash{\widetilde{N}}_n)_{n\geq 0}$
hits~$k$:
$$T_k\,=\,\inf\,\{\,n\geq 0: 
\smash{\widetilde{N}}_n
=k\,\}\,.$$
Let $\cE$ be the event:
$$\cE\,=\,\big\{\,
\forall k<m/\sqrt{\pi}\quad
\smash{\widetilde{N}}_{T_k+1}\geq \rho k\,\big\}\,.$$
We claim that,
on the event~$\cE$, we have
$$\forall n\leq \tau^*\qquad
\smash{\widetilde{N}}_{n+1}\geq \rho 
\smash{\widetilde{N}}_n
\,.$$
Let us prove this inequality by induction on~$n$.
We have $T_1=0$ and
$\smash{\widetilde{N}}_1
>\rho 
\smash{\widetilde{N}}_{0}$, so that
 the inequality is true
for $n=0$. 
Suppose that the inequality has been proved until rank 
$n< \tau^*$, so that
$$\forall k\leq n\qquad
\smash{\widetilde{N}}_{k+1}\geq \rho 
\smash{\widetilde{N}}_{k}\,.$$
This implies in particular that
$$
\smash{\widetilde{N}}_0\,<\,
\smash{\widetilde{N}}_1\,<\,\dots\,<\,
\smash{\widetilde{N}}_n\,<\,m/\sqrt{\pi}\,.$$
Suppose that $
\smash{\widetilde{N}}_n=i$. The above inequalities imply that
$T_i=n$ and
$$\smash{\widetilde{N}}_{T_i+1}=
\smash{\widetilde{N}}_{n+1}\,\geq\,\rho 
\smash{\widetilde{N}}_n\,,$$
so that the inequality still holds at rank $n+1$.
Iterating the inequality until time $\tau^*-1$,
we see that
$$
\smash{\widetilde{N}}_{\tau^*-1}\,\geq\,\rho^{\tau^*-1}\,.$$
Moreover
$\smash{\widetilde{N}}_{\tau^*-1}
\,\leq\,m/\sqrt{\pi}$, thus
$$\tau^*\,\leq\,1+\frac{\ln m}{\ln\rho}\,.$$
Let $m_1\geq 1$ and $\kappa>0$ be such that
$$\forall m\geq m_1\qquad
1+\frac{\ln m}{\ln\rho}\,\leq\,\kappa\ln m\,.$$
The constants $m_1,\kappa$ depend only on $\rho$, and
we have 
$$
P\big(\tau^*\leq\kappa \ln m,\,\tau^*<\tau_0
\,|\,
\smash{\widetilde{N}}_0=1\big)\,\geq\,
P(\cE)\,.$$
We shall use the following lemma to bound
$P(\cE)$ from below.
To avoid too small indices, we write $T(i)$ instead of $T_i$.
\begin{lemma}\label{indepday} 
Let $k\in\um$ and let $i_1,\dots,i_k$ be $k$ distinct points of~$\um$.
The random variables
$\smash{\widetilde{N}}_{T_{i_1}+1},\dots,
\smash{\widetilde{N}}_{T_{i_k}+1}$ are independent.
\end{lemma}
\begin{proof}
We do the proof by induction over~$k$.
For $k=1$, there is nothing to prove.
Let $k\geq 2$ and
suppose that the result has been proved until rank $k-1$.
Let $i_1,\dots,i_k$ be $k$ distinct points of~$\um$.
Let $j_1,\dots,j_k$ be $k$ points of~$\um$.
Let us set
$$T\,=\,\min\,\big\{\,T(i_l):1\leq l\leq k\,\big\}\,.$$
We denote by
$(p(i,j))_{0\leq i,j\leq m}$
the transition matrix of the Markov chain
$(\smash{\widetilde{N}}_n)_{n\geq 0}$.
Using the Markov property, we have
$$\displaylines{
P\big(
\smash{\widetilde{N}}_{T(i_1)+1}=j_1,\dots,
\smash{\widetilde{N}}_{T(i_k)+1}=j_k\big)
\hfill\cr
\,=\,
\sum_{1\leq l\leq k}
P\big(
\smash{\widetilde{N}}_{T(i_1)+1}=j_1,\dots,
\smash{\widetilde{N}}_{T(i_k)+1}=j_k,T=T(i_l)\big)
\cr
\,=\,
\sum_{1\leq l\leq k}
P\big(
\smash{\widetilde{N}}_{T(i_1)+1}=j_1,\dots,
\smash{\widetilde{N}}_{T(i_k)+1}=j_k\,|\,
T=T(i_l)\big)
P\big(T=T(i_l)\big)
\cr
\,=\,
\sum_{1\leq l\leq k}
P\big(\forall h\neq l
\quad
\smash{\widetilde{N}}_{T(i_h)+1}=j_h,\,\smash{\widetilde{N}}_1=j_l
\,|\,
\smash{\widetilde{N}}_0=i_l\big)
P\big(T=T(i_l)\big)
\cr
\,=\,
\sum_{1\leq l\leq k}
p(i_l,j_l)
P\big(\forall h\neq l
\quad
\smash{\widetilde{N}}_{T(i_h)+1}=j_h
\,|\,
\smash{\widetilde{N}}_0=j_l\big)
P\big(T=T(i_l)\big)\,.
}$$
We use the induction hypothesis:
$$P\big(\forall h\neq l
\quad
\smash{\widetilde{N}}_{T(i_h)+1}=j_h
\,|\,
\smash{\widetilde{N}}_0=j_l\big)
\,=\,\prod_{h\neq l}p(i_h,j_h)\,.$$
Reporting in the sum, we get
$$\displaylines{
P\big(
\smash{\widetilde{N}}_{T(i_1)+1}=j_1,\dots,
\smash{\widetilde{N}}_{T(i_k)+1}=j_k\big)
\,=\,\hfill\cr
\,=\,
\sum_{1\leq l\leq k}
\prod_{1\leq h\leq k}p(i_h,j_h)
P\big(T=T(i_l)\big)
\,=\,
\prod_{1\leq h\leq k}p(i_h,j_h)
\,.
}$$
This completes the induction step and the proof.
\end{proof}

\noindent
Using 
lemma~\ref{indepday} and
proposition~\ref{crqa}, we obtain, for $m$ larger than
$m_0$ and $m_1$,
\begin{align*}
P(\cE)\,&\geq\,
\prod_{1\leq k\leq m}
P\big(
\smash{\widetilde{N}}_{T_k+1}\geq\rho k)\cr
&\,=\,
\prod_{1\leq k\leq m}
\Big(1-P\big(\,
N_{1}<\rho k\,\big|\,N_0=k\,\big)\Big)\cr
&\,\geq\,
\prod_{1\leq k\leq m}
\Big(1- \exp(-c_0k) \Big)
\,\geq\,
\prod_{k=1}^{\infty}
\Big(1- \exp(-c_0k) \Big)
\,.
\end{align*}
The last infinite product is converging. 
Let us denote
its value by $p_1$. 
Let also
$$p_2\,=\,
\min\,\Big\{\,
P\big(\tau^*\leq\kappa \ln m\,|\,N_0=1\big)
:m\leq \max(m_0,m_1)\,\Big\}
.$$
The value $p_2$ is positive
and the inequality stated in the proposition
holds with $p^*=\min(p_1,p_2)$.
\end{proof}
\section{Conclusion}
Our goal is to put forward the importance of the parameter~$\pi$.
To this end,
we have studied the behavior of the simple genetic algorithm in two contrasting regimes. 
In the first case, we take $\ell=m$ and we run the genetic algorithm on the sharp peak
landscape with an initial population containing exactly one Master sequence and $m-1$ chromosomes
very far from it.
The parameters of the genetic algorithm are set so that $\pi<1$. 
We showed that the Master sequence
is very likely to be lost and that the mean fitness does not increase significantly.
In the second case, we consider an arbitrary fitness landscape and we start with a population such that
$\pi>1$. 
We showed that the mean fitness is likely to increase.
From these results, we extrapolate a simple practical rule. 
We believe that the parameters of the genetic algorithm sould be tuned so that
$\pi$
is slightly larger than~$1$, that is, at each generation, we should have
$$
\text{maximal fitness}\times (1-p_C)(1-p_M)^\ell\,>\,
\text{mean fitness}
\,.$$
For instance, one could use an adaptive scheme of the parameters, as
suggested in \cite{CGA}.
Of course this conclusion has to be taken with care. 
We hope that it will
be further examined in future research works.
On the empirical side, it should be tested 
numerically on various problems.
On the theoretical side, it might be extended to variants of the simple genetic algorithm,
as well as to the evolutionary computation framework.
\medskip

\noindent
{{\bf{Acknowledgments.}}
I thank two anonymous Referees for their careful reading and their remarks,
which helped to improve the presentation.

\appendix
\section{Appendix}

\noindent
{\bf Monotonicity}.
We first recall some standard definitions concerning monotonicity 
and coupling
for stochastic processes.
A classical reference
is Liggett's book
\cite{LIG},
especially for applications to particle systems. 
In the next two definitions,
we consider
a discrete time Markov chain 
$(X_n)_{n\geq 0}$ 
with
values in a space $\cE$.
We suppose that the state space $\cE$ is finite and
that it is equipped with a partial order $\leq$.
A function $f:\cE\to\R$ is non--decreasing if
$$\forall x,y\in\cE\qquad
x\leq y\quad\Rightarrow\quad f(x)\leq f(y)\,.$$
\begin{definition}\label{defmo}
The Markov chain
$(X_n)_{n\geq 0}$ is said to be monotone if, 
for any non--decreasing function $f$, the function
$$x\in\cE\mapsto E\big(f(X_n)\,|\,X_0=x\big)$$
is non--decreasing.
\end{definition}
A natural way to prove monotonicity is to construct an adequate coupling.
A coupling 
for the Markov chain
$(X_n)_{n\geq 0}$ 
is a family of processes
$(X_n^x)_{n\geq 0}$
indexed by 
$x\in\cE$, which are all defined on the same probability space, and such that, 
for $x\in\cE$, the process
$(X_n^x)_{n\geq 0}$ is the Markov chain 
$(X_n)_{n\geq 0}$ 
starting from $X_0=x$.
The coupling is said to be monotone if
$$\forall x,y\in\cE\qquad
x\leq y\quad\Rightarrow\quad \forall n\geq 1\qquad X_n^x\leq X_n^y\,.$$
%
%
%
If there exists a monotone coupling, 
then the 
Markov chain
is monotone.
\medskip

\noindent
{\bf Stochastic domination}.
Let $\mu,\nu$ be two probability measures on $\mathbb R$.
We say that $\nu$ stochastically dominates $\mu$,
which we denote by $\mu\preceq\nu$, if
for any non--decreasing positive function $f$, we have
$\mu(f)\leq \nu(f)$.
\begin{lemma}\label{stoint}
If $\mu,\nu$ are two probability measures on $\mathbb N$,
then 
$\mu$ is stochastically dominated by $\nu$ if and only if
$$\forall i\in {\mathbb N}\qquad
\mu([i,+\infty[)\,\leq\,
\nu([i,+\infty[)\,.$$
\end{lemma}
\begin{proof}
Let $f:\N\to\R^+$ be a non--decreasing function. We compute
$$\displaylines{
\mu(f)\,=\,\sum_{i\geq 0}\mu(i)f(i)
\,=\, \sum_{i\geq 0}\big(
\mu([i,+\infty[)-
\mu([i+1,+\infty[)\big)
f(i)\cr
\,=\,f(0)+\sum_{i\geq 1}
\mu([i,+\infty[)(f(i)-f(i-1))\,.
}$$
Under the above hypothesis, we conclude indeed that
$\mu(f)\leq\nu(f)$.
\end{proof}
\begin{lemma}\label{binopoi}
Let $n\geq 1$, $p\in [0,1]$, $\lambda>0$ be such that
$(1-p)^n\geq\exp(-\lambda)$.
Then
the binomial law
$\cB(n,p)$ of parameters $n,p$
is stochastically dominated by the Poisson 
law $\cP(\lambda)$ of parameter $\lambda$.
\end{lemma}
\begin{proof}
Let $X_1,\dots,X_n$ be independent random variables
with common law the Poisson law of parameter $-\ln(1-p)$.
Let $Y$ be a further random variable, independent of $X_1,\dots,X_n$,
with law 
the Poisson law of parameter $\lambda-n\ln(1-p)$.
Obviously, we have
$$Y+X_1+\dots+X_n\,\geq\,\min(X_1,1)+\dots+\min(X_n,1)\,.$$
Moreover, the law of the lefthand side is the Poisson
law of parameter~$\lambda$, while the law of the righthand side
is the binomial law $\cB(n,p)$.
\end{proof}
\begin{lemma}\label{poisstail}
Let $\lambda>0$ and let $Y$ be a random variable with law the 
Poisson law $\cP(\lambda)$ of parameter~$\lambda$.
For any $t\geq\lambda$, we have
$$P(Y\geq t)\,\leq\,\Big(\frac{\lambda e}{t}\Big)^t\,.$$
\end{lemma}
\begin{proof}
We write
\begin{multline*}
P(Y\geq t)\,=\,
\sum_{k\geq t}
\frac{\lambda^k}{k!}\exp(-\lambda)
\,=\,
\sum_{k\geq t}
\frac{\lambda^{k-t}}{k!}\exp(-\lambda)\lambda^t
\cr
\,\leq\,
\sum_{k\geq t}
\frac{t^{k-t}}{k!}\exp(-\lambda)\lambda^t
\,\leq\,
\Big(\frac{\lambda e}{t}\Big)^t\,.
\end{multline*}
\end{proof}

\noindent
Let $Y$ be a random variable following the Poisson law $\cP(\lambda)$.
For any $t\in\R$, we have
$$\Lambda_Y(t)\,=\,\ln E\big(\exp(tY)\big)\,=\,
\ln\Big(
\sum_{k=0}^\infty
\frac{\lambda^k}{k!}\exp(-\lambda+kt)\Big)\,=\,
\lambda \big(\exp(t)-1\big)\,.$$
For any $\alpha,t\in\R$, 
$$\Lambda_{\alpha Y}(t)\,=\,
\Lambda_{Y}(\alpha t)
\,=\,\lambda \big(\exp(\alpha t)-1\big)\,.$$
Let us compute
the Fenchel--Legendre transform
$\Lambda_{\alpha Y}^*$.
By definition, for $x\in\R$,
$$\Lambda_{\alpha Y}^*(x)\,=\,
\sup_{t\in\R}\Big(tx-
\lambda \big(\exp(\alpha t)-1\big)\Big)\,.$$
The maximum is attained at $t=(1/\alpha)\ln (x/(\lambda\alpha))$,
hence
$$\Lambda_{\alpha Y}^*(x)\,=\,
\frac{x}{\alpha}\ln\Big(\frac{x}{\lambda\alpha}\Big)
-
\frac{x}{\alpha}+\lambda\,.$$
\begin{lemma}
\label{bnpp}
Let $p\in[0,1]$ and let $n\geq 1$. Let $X$ be a random variable following
the binomial law $\cB(n,p)$. 
Let $Y$ be a random variable following
the Poisson law $\cP(np)$. For any $\alpha\in\R$,
we have
$\Lambda^*_{\alpha X}\geq \Lambda^*_{\alpha Y}$.
\end{lemma}
\begin{proof}
For any $t\in\R$, we have
$$\Lambda_X(t)\,=\,\ln E\big(\exp(tX)\big)\,=\,
n\ln\big(1-p+p\exp(t)\big)\,\leq\,np\big(\exp(t)-1\big)\,.$$
For any $\alpha,t\in\R$, 
$$\Lambda_{\alpha X}(t)\,=\,
\Lambda_{X}(\alpha t)
\,\leq\,np\big(\exp(\alpha t)-1\big)\,.$$
We recall that, if $Y$ is distributed according to the 
Poisson law of parameter~$\lambda$, then
$$\forall t\in\R\qquad
\Lambda_{Y}(t)\,=\,\lambda(\exp(t)-1)\,.$$
Thus, taking $\lambda=np$, we conclude that
$$\forall t\in\R\qquad
\Lambda_{\alpha X}(t)\,\leq\, \Lambda_{\alpha Y}(t)\,.$$
Taking the Fenchel--Legendre transform, we obtain
$$\forall x\in\R\qquad
\Lambda^*_{\alpha X}(x)\,\geq\, \Lambda^*_{\alpha Y}(x)\,$$
as required.
\end{proof}

\noindent
{\bf Hoeffding's inequality}.
We state Hoeffding's inequality for Bernoulli random variables
\cite{HOE}.
Suppose that $X$ is a random variable with law the binomial law
$\cB(n,p)$. We have
$$\forall t<np\qquad P(X<t)\,\leq\,\exp\Big(-\frac{2}{n}\big(np-t)^2\Big)\,.
$$
\noindent
{\bf Chebyshev exponential inequality}.
Let $X_1,\dots,X_n$ be i.i.d. random variables with common law~$\mu$.
Let $\Lambda$ be the Log--Laplace of~$\mu$, defined by
$$\forall t\in\R\qquad\Lambda(t)\,=
\ln\Big(\int_{\mathbb R}\exp(ts)\,d\mu(s)\Big)\,.$$
Let $\Lambda^*$ be the Cram\'er transform of~$\mu$, defined by
$$\forall x\in\R\qquad
\Lambda^*(x)\,=\,\sup_{t\in\R}\big(tx-\Lambda(t)\big)
\,.$$
We suppose that $\mu$ is integrable and we denote by $m$ its mean,
i.e., $m=\int_\R x\,d\mu(x)$. 
We have then 
(see for instance \cite{DZ})
$$\forall x\geq m\qquad
P\Big(\frac
{1}{n}\big(
{X_1+\cdots+X_n}\big)\,\geq\,x\Big)\,\leq\,\exp\big(-n\Lambda^*(x)\big)\,.$$
\noindent
Let $Y$ be a random variable following the Poisson law $\cP(\lambda)$.
For any $t\in\R$, we have
$$\Lambda_Y(t)\,=\,\ln E\big(\exp(tY)\big)\,=\,
\ln\Big(
\sum_{k=0}^\infty
\frac{\lambda^k}{k!}\exp(-\lambda+kt)\Big)\,=\,
\lambda \big(\exp(t)-1\big)\,.$$
For any $\alpha,t\in\R$, 
$$\Lambda_{\alpha Y}(t)\,=\,
\Lambda_{Y}(\alpha t)
\,=\,\lambda \big(\exp(\alpha t)-1\big)\,.$$
Let us compute
the Fenchel--Legendre transform
$\Lambda_{\alpha Y}^*$.
By definition, for $x\in\R$,
$$\Lambda_{\alpha Y}^*(x)\,=\,
\sup_{t\in\R}\Big(tx-
\lambda \big(\exp(\alpha t)-1\big)\Big)\,.$$
The maximum is attained at $t=(1/\alpha)\ln (x/(\lambda\alpha))$,
hence
$$\Lambda_{\alpha Y}^*(x)\,=\,
\frac{x}{\alpha}\ln\Big(\frac{x}{\lambda\alpha}\Big)
-
\frac{x}{\alpha}+\lambda\,.$$
\noindent
{\bf Galton--Watson processes}.
Let $\nu$ be probability distribution on the non--negative integers.
Let 
$(Y_n)_{n\in\mathbb N}$ be a sequence of i.i.d. random variables distributed
according to~$\nu$. The Galton--Watson process with reproduction law~$/nu$ is the sequence 
of random variables
$(Z_n)_{n\in\mathbb N}$ defined by $Z_0=1$ and
$$\forall n\in{\N}\qquad
Z_{n+1}\,=\,\sum_{k=1}^{Z_n}Y_k\,.$$
It is said to be subcritical if $E(\nu)<1$ and supercritical if $E(\nu)>1$.
The following estimates are classical 
(see for instance~\cite{AN}).
\begin{lemma}\label{subgw}
Let
$(Z_n)_{n\in\mathbb N}$ be a subcritical Galton--Watson process.
There exists
a positive constant $c$, which depends only on the law~$\nu$,
such that
$$\forall n\geq 1\qquad P\big(Z_n>0\big)\,\leq\,\exp(-cn)\,.$$ 
\end{lemma}
\begin{proposition}\label{boundtau}
Let 
$(Z_n)_{n\in\mathbb N}$
be a supercritical Galton--Watson process such that $E(\nu)$ is finite.
Let
$$\tau_1\,=\,\inf\,\big\{\,n\geq 1: Z_n>n^{1/4}\,\big\}\,.$$
There exist $\kappa>0$, $c_1>0$, $n_1\geq 1$, such that
$$\forall n\geq n_1\qquad
P\big(\tau_1<\kappa\ln n\big)\,\leq\,\frac{1}{n^{c_1}}\,.$$
\end{proposition}
\begin{proof}
We have, for $k\geq 0$,
\begin{multline*}
P(\tau_1=k)\,\leq\,
P\big(\tau_1\geq k,Z_k>n^{1/4}\big)
\cr
\,\leq\,
P\big(Z_k>n^{1/4}\big)\,\leq\,
n^{-1/4}E\big(Z_k\big)\,\leq\,
n^{-1/4}(E(\nu))^k\,.
\end{multline*}
We sum this inequality: for $n\geq 1$,
$$P(\tau_1<n)\,\leq\,
n^{-1/4}\sum_{k=0}^{n-1}(E(\nu))^k\,=\,
n^{-1/4}\frac{(E(\nu))^n-1}{E(\nu)-1}\,.$$
We choose $\kappa$ 
positive and sufficiently small, we apply this inequality with $\kappa\ln n$
instead of $n$ and we obtain the desired
conclusion.
\end{proof}

\bibliographystyle{plain}
\bibliography{sga}
 \thispagestyle{empty}

\end{document}